\documentclass[nohyperref]{article}

\usepackage{microtype}
\usepackage{graphicx}
\usepackage{subfigure}
\usepackage{booktabs}
\usepackage{hyperref}

\usepackage[accepted]{icml2023}

\usepackage{amsmath}
\usepackage{amssymb}
\usepackage{mathtools}  
\mathtoolsset{showonlyrefs}
\usepackage{amsthm}

\theoremstyle{plain}
\newtheorem{thm}{Theorem}[section]

\newtheorem{Prop}[thm]{Proposition}

\theoremstyle{definition}
\newtheorem{Def}[thm]{Definition}
\theoremstyle{remark}
\newtheorem{Rem}[thm]{Remark}

\usepackage[textsize=tiny]{todonotes}

\usepackage{bm}

\newcommand{\md}{{\mathrm{d}}}
\newcommand{\mR}{{\mathbb{R}}}
\newcommand{\mZ}{{\mathbb{Z}}}
\newcommand{\mL}{{\mathcal{L}}}

\newcommand{\trans}{\top}

\DeclareMathOperator{\Span}{span}

\usepackage{multirow}
\usepackage{makecell}
\usepackage{graphicx}
\usepackage{marvosym}

\usepackage{algorithm}
\usepackage{algorithmic}

\icmltitlerunning{Learning to simulate partially known spatio-temporal dynamics with trainable difference operators}

\begin{document}

\twocolumn[

    \icmltitle{Learning to simulate partially known spatio-temporal dynamics \\ with trainable difference operators}



    \icmlsetsymbol{equal}{*}

    \begin{icmlauthorlist}
        \icmlauthor{Xiang Huang}{equal,xxx,comp}
        \icmlauthor{Zhuoyuan Li}{equal,yyy}
        \icmlauthor{Hongsheng Liu}{comp}
        \icmlauthor{Zidong Wang}{comp}
        \icmlauthor{Hongye Zhou}{comp}
        \icmlauthor{Bin Dong \textsuperscript{\Letter}}{aaa,bbb}
        \icmlauthor{Bei Hua \textsuperscript{\Letter}}{xxx}
    \end{icmlauthorlist}

    \icmlaffiliation{xxx}{University of Science and Technology of China}
    \icmlaffiliation{yyy}{Peking University}
    \icmlaffiliation{aaa}{Beijing International Center for Mathematical Research, Peking University}
    \icmlaffiliation{bbb}{Center for Machine Learning Research, Peking University}
    \icmlaffiliation{comp}{Huawei Technologies Co. Ltd}

    \icmlcorrespondingauthor{Bin Dong}{dongbin@math.pku.edu.cn}
    \icmlcorrespondingauthor{Bei Hua}{bhua@ustc.edu.cn}

    \icmlkeywords{Machine Learning, ICML}

    \vskip 0.3in
]

\printAffiliationsAndNotice{\icmlEqualContribution} 

\begin{abstract}
    Recently, using neural networks to simulate spatio-temporal dynamics has received a lot of attention. However, most existing methods adopt pure data-driven black-box models, which have limited accuracy and interpretability. By combining trainable difference operators with black-box models, we propose a new hybrid architecture explicitly embedded with partial prior knowledge of the underlying PDEs named PDE-Net++. Furthermore, we introduce two distinct options called the trainable flipping difference layer (TFDL) and the trainable dynamic difference layer (TDDL) for the difference operators. Numerous numerical experiments have demonstrated that PDE-Net++ has superior prediction accuracy and better extrapolation performance than black-box models.
\end{abstract}

\section{Introduction} \label{sec:intro}
Simulating complex spatio-temporal dynamics governed by partial differential equations (PDEs) has been a core of physical science and engineering.
Once the PDE that describes the dynamical system is explicitly established, classical numerical solvers such as Finite Difference Methods (FDMs) \cite{LeVeque2007FDM,Thomas2013NPDE-FDM} and Finite Volume Methods (FVMs) \cite{LeVeque2002FVM} which have been well-developed over the past few decades are usually applied for simulation via suitable discretizations on the grids. These methods are widely used because of their reliable theoretical analysis on convergence and stability.

Nevertheless, lots of dynamical systems appearing in reality such as turbulence modeling and weather forecasting are partially known. Turbulent flows are characterized by multiscale spatio-temporal chaos, and there is no strict separation between low-frequency coherent dynamics and turbulent fluctuations. To avoid the unaffordable computational costs, governing equations for filtered or averaged variables are established to study the statistical features, but the influence of the unresolved scales cannot be determined, which still leads to the closure problem \cite{Pope2000Turbulent}. Things become even more difficult in the field of modern numerical weather prediction (NWP) due to limit of computational resources and lack of obervation data. Apart from a dynamical core describing the atmosphere in terms of momentum, mass and enthalphy \cite{Jacobson2005,Ulrich2022,IFSdocDynamics}, a well-designed NWP model needs physical parameterizations for the subgrid-scale processes such as turbulent transport near the surface, cumulus convection, microphysics and radiation \cite{IFSdocPhysPara}. Unfortunately, these subgrid-scale processes are so complicated that there exist a large number of physical features unable to be modeled via explicit formulas.

In the field of deep learning, these kinds of spatio-temporal dynamical systems are often simulated by neural networks.
Some studies make use of either convolutional neural networks (CNNs) \cite{guo2016convolutional, YinhaoZhu2018BayesianDC, YaserAfshar2019PredictionOA, YuehawKhoo2021SolvingPP, PranshuPant2021DeepLF} or graph neural networks (GNNs) \cite{FrancisOgoke2021GraphCN, ZijieLi2022GraphNN, li2022graph} to learn the spatial relationships in a PDE based on the form of mesh grids, and then employ various schemes to evolve over time.
Meanwhile, the neural operator \cite{LuLu2021LearningNO,li2020fourier,guibas2021efficient, tran2021factorized,ShuhaoCao2021ChooseAT} directly learns the solution mapping between two infinite-dimensional function spaces. This success in various PDE prototypes has received significant attention from scientists studying atmospheric modeling \cite{Dueben2018Challenges}, and many network architectures have been adapted for use in atmospheric science for weather prediction \cite{Schultz2021CanDLbeat}. For example, FourCastNet \cite{pathak2022fourcastnet}, Pangu-Weather \cite{bi2022pangu}, and GraphCast \cite{lam2022graphcast} have demonstrated great potential in weather forecasting scenarios. Despite the vast potential of these pure data-driven black-box models, their high accuracy is dependent on excessive training costs and a large amount of training data. Additionally, these models do not take into account physics knowledge such as conservation laws and dynamic equations, raising concerns about interpretability.

Recently, more and more attention has been focused on embedding partial prior knowledge into network architectures. \cite{FilipedeAvilaBelbutePeres2020CombiningDP} establishes a hybrid neural network combining traditional graph convolutional networks with an embedded differentiable fluid dynamics simulator inside the network. Inspired by the idea of Large Eddy Simulation (LES) for turbulence, TF-Net \cite{Wang2020TowardsPIDLforTurbulence} tries to decompose the turbulent flow with respect to different levels of energy spectrum via spatial and temporal filters, while JAX-CFD \cite{Dmitrii2021MLaccCFD} contains learnable interpolation operators for advected and advecting velocity components and a pressure projection to enforce the zero-divergence condition. Such methods are problem-specific and thus difficult to be applied to other tasks. Another more universal kind of approaches try to directly resemble the associated differential operators. For instance, to address inverse problems, PDE-FIND \cite{Samuel2017PDE-FIND} takes the advantage of polynomial interpolation rather than directly ultilizes fixed finite-difference approximations. Alternatively, PDE-Net \cite{long2018pde,long2019pde} followed by DSN \cite{So2021DSN} substitutes the spatial differential operators with moment-constrained convolutions. Each convolutional operator is guaranteed to approximate a fixed differentiation with learnable parameters, which shows stability and high accuracy on long-term prediction tasks.

\paragraph{Our contributions.} In this paper, we propose a hybrid  network architecture ``PDE-Net++'' for partially known spatio-temporal dynamical systems. Specifically, PDE-Net++ incorporates possibly trainable difference operators for all the related derivatives appearing in the known part and a purely data-driven neural network for the unknown part. Motivated by the delicate numerical schemes as well as the dynamic convolution \cite{De2016DynamicFilter,Su2019Pixel,Han2021DynamicFilterSurvey}, we replace the original universal moment-constrained kernel in PDE-Net with a kernel for each grid point that encodes local features, which lead to the newly proposed TFDL and TDDL modules. PDE-Net++ is capable of performing long-period fast and stable simulation with high accuracy once trained with a few number of observation data. The main contributions are summarized as follows:
\begin{itemize}
    \item A new hybrid architecture named PDE-Net++ is proposed, which effectively combines difference operators and black-box models, essentially realizing the explicit encoding of the partially prior knowledge of the underlying PDEs.
    \item Within the PDE-Net++ architecture, two new modules named the TFDL and the TDDL are first given in this paper based on the moment-constrained convolution, which generally show better performances in comparison of the existing well-designed difference schemes and moment-constrained convolutions.
    \item Extensive numerical experiments have been performed to demonstrate the effectiveness of PDE-Net++, which indicate that combining physical priors with neural networks have significantly higher prediction accuracy than the pure black-box models for both interpolation and extrapolation.
\end{itemize}

\section{Methodology}
We focus on solving spatio-temporal PDEs in the general form of
\begin{subequations}\label{eq:def}
    \begin{align}
         & \frac{\partial\bm U}{\partial t} = \mL(\bm x, \bm U), \quad (\bm x,t) \in \Omega \times [0,T], \label{eq:def_a} \\
         & \mathcal{I}(\bm x, \bm U) = 0, \quad \bm x \in \Omega, \label{eq:def_b}                                         \\
         & \mathcal{B}(\bm x, t, \bm U) = 0, \quad (\bm x,t) \in \partial \Omega \times [0,T] \label{eq:def_c}
    \end{align}
\end{subequations}
for $n$-dimensional state variables $\bm U=\bm U(\bm x,t):\Omega\times[0,T]\to\mR^n$. Here, $\mL(\cdot)$ stands for the differential operator describing the underlying dynamics with partial knowledge, which is furthermore decomposed into two parts, namely the known part $\mL_{\textrm{known}}$ and the unknown part $\mL_{\textrm{unknown}}$.
\[\mL(\bm x, \bm U)=\mL_{\text{known}}(\bm x, \bm U)+\mL_{\text{unknown}}(\bm x, \bm U).\]
The operators $\mathcal{I}(\cdot)$ and $\mathcal{B}(\cdot)$ correspond to the initial and the boundary conditions, respectively. Assume that we have an explicit formula
\begin{equation}\label{eq:Lknown-decomp}
    \mL_{\textrm{known}}(\bm x,\bm U)=\bm\Phi(\bm x,\bm U,\nabla\bm U,\nabla^2\bm U,\cdots)
\end{equation}
for the known part, but no prior knowledge is provided for the unknown one. Our task is to develop a surrogate model that recovers the hidden dynamical behavior based on noisy observations of $\bm U$. Once such model is established, it is capable of generating fast simulation starting from new initial conditions with high accuracy.

\subsection{The PDE-Net++ architecture}
We employ the forward Euler integration scheme
\begin{equation}\label{eq:forward-euler-scheme}
    \resizebox{.9\hsize}{!}{$%
            \begin{aligned}
                \bm U_{j+1} & \approx\bm U_j+\mL(\bm x,\bm U_j)\Delta t                                                          \\
                            & =\bm U_j+\mL_{\textrm{known}}(\bm x,\bm U_j)\Delta t+\mL_{\textrm{unknown}}(\bm x,\bm U_j)\Delta t
            \end{aligned}
        $%
    }%
\end{equation}
to discretize Eq. \eqref{eq:def_a}\noeqref{eq:def_b}\noeqref{eq:def_c} along the temporal dimension, where the subscript $j$ stands for the time index with a time interval $\Delta t$. The general idea of PDE-Net++ is to approximate the known part with certain finite difference operators, and to approximate the unknown part with a neural network $\mathcal{F}_{\text{NN}}(\bm x, \bm U_j)$ named as the ``\textit{backbone}''. The updating formula of PDE-Net++ according to Eq. \eqref{eq:forward-euler-scheme} is
\begin{equation}\label{eq:my_forward_euler}
    \begin{aligned}
        \hat{\bm U}_{j+1} & =\bm U_j+\hat{\bm\Phi}(\bm x,\bm U_j)\Delta t+\mathcal{F}_{\mathrm{NN}}(\bm x, \bm U_j)\Delta t,
    \end{aligned}
\end{equation}
where
\[\hat{\bm\Phi}(\bm x,\bm U_j)=\bm\Phi(\bm x,\bm U_j, D(\bm U_j), D^2(\bm U_j),\cdots),\]
and the symbol $\hat{\bm U}$ stands for the model prediction. The derivative terms such as $\nabla\bm U$ and $\nabla^2\bm U$ in $\bm\Phi$ are replaced by the corresponding difference operators $D(\bm U)$ and $D^2(\bm U)$ respectively, which can be implemented via existing stable finite difference schemes or the trainable difference layers introduced later. Besides, the neural network $\mathcal{F}_{\mathrm{NN}}$ is designed to account for the truncation error introduced by the former difference operators \cite{Um2020SolverintheLoop} as well as the remaining unknown part $\mL_{\textrm{unknown}}(\bm x, \bm U_j)$.

\begin{figure*}
    \begin{center}
        \includegraphics[width=1.0\textwidth]{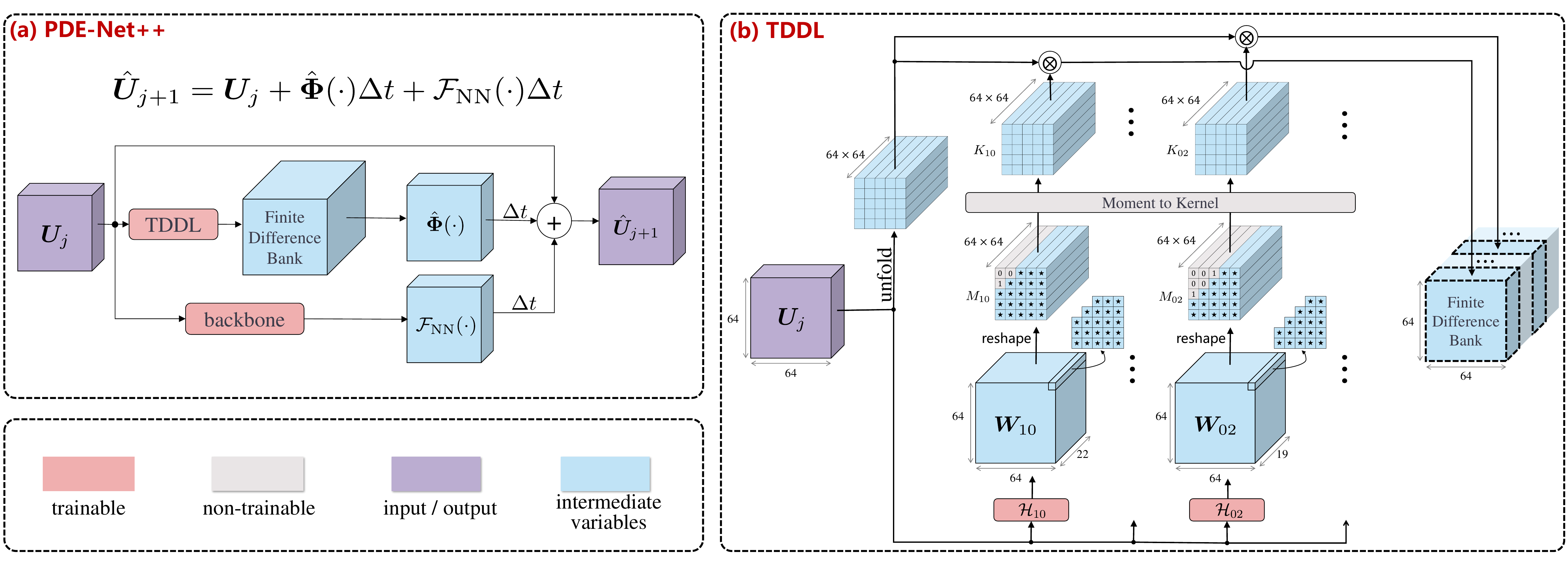}
        \caption{(a) The overall architecture of PDE-Net++ together with the TDDL module. (b) The trainable dynamic difference layer (TDDL). For illustration, the mesh size is set as $64\times64$. Both the kernel size and the moment size are fixed as $5\times5$, and the difference operators corresponding to the derivatives $\partial/\partial x$ and $\partial^2/\partial y^2$ are displayed as an example.}
        \label{fig:method_pdenet_plus}
    \end{center}
\end{figure*}

Fig.\ref{fig:method_pdenet_plus}(a) shows the overall architecture of PDE-Net++ with the difference operators realized by the TDDL module as an example. The input $\bm U_j$ records the state variables (possibly predictions from the previous data) at time step $j$, which is then fed separately to the TDDL module and the backbone. The TDDL module produces approximations of derivatives of $\bm U_j$ and saves them in a finite difference bank. Afterwards, the results of finite differences are collected to compute $\hat{\bm\Phi}(\bm x,\bm U_j)$ as an approximation of $\mL_{\textrm{known}}(\bm x,\bm U_j)$. Meanwhile, a backbone network processing on $\bm U_j$ produces $\mathcal{F}_{\textrm{NN}}(\bm x,\bm U_j)$ to approximate the $\mL_{\textrm{unknown}}$ term. Finally, the respective results of the two paths multiplied by $\Delta t$ give rise to the predictions for the next time step via the ``skip-connection'' introduced in ResNet \cite{he2016deep}. Regarding the choices of the backbone, we test some classic models including U-Net \cite{OlafRonneberger2015UNetCN}, FNO \cite{li2020fourier} together with its variant F-FNO \cite{tran2021factorized}, ConvResNet \cite{he2016deep}, and Galerkin Transformer \cite{ShuhaoCao2021ChooseAT}.

\subsection{Trainable difference layers}\label{sec:difflayers}
In order to introduce our physics prior, the ideas borrowed from FDMs are used to resemble the derivatives in $\mL_{\textrm{known}}$. For any 1D scalar function $f:\mR\to\mR$ and a set of $N_P$ different grid points $\{x_l\}_{l=1}^{N_P}$, the (possibly high-order) derivative at some point $\bar x$ can be approximated via a linear combination
\[\left.\frac{\md^kf}{\md x^k}\right|_{x=\bar x}\approx\sum_{l=1}^{N_P}\alpha_l^{(k)}f(x_l)\]
if the coefficients $\alpha_l^{(k)}$ satisfy a non-homogeneous linear system depending on the order $k$ and the relative distances $(x_l-\bar x)$ \cite{fornberg1988generation}. Similar deductions can be applied to high-dimensional cases. Note that such linear systems have solution sets that can be easily parameterized, which allows us to introduce tunable parameters when implementing the difference operators.
\begin{Def}
    A parameterized map $D_\theta:\mZ^2\times\mR_+^2\to\mZ^2$ is defined as a \textit{difference layer} corresponding to the differential operator $\frac{\partial^{p+q}}{\partial x^p\partial y^q}$ with a parameter space $\Theta$ if for any smooth function $h$ and $\theta\in\Theta$,
    \[D_\theta(G(h);\Delta x,\Delta y)(k,l)\to G\left(\frac{\partial^{p+q}h}{\partial x^p\partial y^q}\right)(k,l)\]
    for each $(k,l)\in\mZ^2$ as $\Delta x,\Delta y\to0$, where the regular grid sampler $G$ is defined as
    \[G(h)(k,l):=h(k\Delta x,l\Delta y).\]
\end{Def}
\begin{Rem}
    Indeed, the schemes coming from FDMs or FVMs can be viewed as difference layers with trivial parameter spaces (containing only one point). Our convention for difference layers excludes such cases to distinguish trainable difference layers from non-trainable ones.
\end{Rem}

Inspired by the systematic connection between spatially differential operators and convolutional kernels established by \cite{Cai2012ImageRestoration, dong2017ImageRestoration} within the image restoration setting, PDE-Net \cite{long2018pde,long2019pde} essentially gives a convolutional version of the difference layer defined above, which can be considered as a simple special case. 
In fact, PDE-Net gives the following proposition. Note that from now on, we omit the dependencies on $\Delta x$ and $\Delta y$ for simplicity.

\begin{Prop}
    It can be shown by proper truncation (Prop. \ref{prop:moment-constraints}) on the Taylor's expansion that the convolution
    \begin{equation}\label{eq:conv-difflayer-concept}
        \begin{aligned}
            D(G(h))(k,l) & =(K\circledast G(h))(k,l)            \\
                         & =\sum_{s,t\in\mZ}K(s,t)G(h)(k+s,l+t)
        \end{aligned}
    \end{equation}
    with a locally-supported convolutional kernel $K$ is a difference layer corresponding to the differential operator $\frac{\partial^{p+q}}{\partial x^p\partial y^q}$ with a convergence order $(r+1)$ if and only if the moment matrix of $K$
    \begin{equation}
        M(K)(u,v):=\sum_{s,t\in\mZ}K(s,t)\frac{s^ut^v}{u!v!}(\Delta x)^u(\Delta y)^v
    \end{equation}
    satisfies the moment constraint
    \begin{equation}\label{eq:moment-constraints}
        \begin{aligned}
            M(K)(u,v) = \delta_{u,p}\delta_{v,q}\,\textrm{ for all }\, u+v\le p+q+r
        \end{aligned}
    \end{equation}
    with the Kronecker delta $\delta$.
\end{Prop}

Furthermore, if both the kernel size and the moment matrix size are fixed as $(2L+1)\times(2L+1)$, then there is a one-to-one correspondence between the moment matrix and the kernel. In other words, for fixed $(p,q)$ and $r$, any convolutional difference layer can be uniquely determined and parameterized by the free elements $\{M(K)(u,v)\}_{u+v>p+q+r}$ in the moment matrix of its convolutional kernel. Consequently, all convolutional difference layers are of the form
\begin{equation}\label{eq:ConvDiffLayer}
    D_\theta(G(h))(k,l)=\sum_{s,t\in\mZ}\bar K_\theta(s,t)G(h)(k+s,l+t),
\end{equation}
where $\bar K_\theta$ stands for the convolutional kernels parameterized by the free elements of its moment matrix.
\subsection{Trainable flipping difference layers} \label{sec:TFDL}
Nevertheless, most classical numerical schemes do not fix the coefficients across the whole solution domain, which usually rely on the local features of the numerical solution. For instance, upwind schemes originated in \cite{Courant1952Upwind} switch among different sets of coefficients according to the signs of the advection velocities. FVMs usually adapt the flux limiter \cite{Roe1986minmod,VanLeer1974,VanLeer1979} to exchange some accuracy for stability. Besides, ENO and WENO methods \cite{Liu1994WENO,Shu1998ENOandWENO} seek for a (weighted) average of interpolation polynomials based on a heuristic smoothness measure for better reconstructions. One may refer to Appendix \ref{sec:schemes-with-local-features} for more details. These methods all result in different coefficients for the stencils at different grid points. Therefore, it is unwise to expect a universal finite difference operator that can achieve a good balance between accuracy and stability for each grid point simultaneously.

We focus on those difference layers where $D_\theta(G(h))(k,l)$ is defined as
\begin{equation}\label{eq:FlexDiffLayers}
    \sum_{s,t\in\mZ}K_\theta(s,t;\bm U,k,l)G(h)(k+s,l+t),
\end{equation}
where $K_\theta$ may depend on not only $h$ itself but also all other features associated with the differential equations. Such a formula allows $K_\theta$ to choose different coefficients according to the locations and the local features of $h$.

To mimic the upwind behavior, we propose the \textit{Trainable Flipping Difference Layer} (TFDL) for the first-order derivatives to switch between two sets of parameters for $K_\theta$ according to the signs of the coefficients before the derivatives. For instance, suppose that the known part $\mL_{\textrm{known}}$ contains the term ``$v\frac{\partial u}{\partial x}$'', the TFDL for the derivative $\frac{\partial u}{\partial x}$ is defined as
\begin{equation}\label{eq:FlippingKernels}
    \resizebox{.88\hsize}{!}{$%
            K_\theta(s,t;\bm U,k,l)=
            \begin{cases}
                \bar K_\theta(s,t),   & G(v)(k,l)>0, \\
                -\bar K_\theta(-s,t), & G(v)(k,l)<0.
            \end{cases}
        $%
    }%
\end{equation}
As for the derivative $\frac{\partial}{\partial y}$, the rule follows in a similar way, but the sign of $t$ rather than that of $s$ is flipped. It can be checked with little difficulty (Prop. \ref{prop:flipped-kernels}) that the TFDL (\refeq{eq:FlexDiffLayers}, \refeq{eq:FlippingKernels}) shares the same corresponding derivative and the same convergence order with the convolutional difference layer (\refeq{eq:ConvDiffLayer}).

\subsection{Trainable dynamic difference layers}\label{sec:TDDL}
Sometimes the dynamics is too complicated to be resolved by the convolutional difference layer or the TFDL. Motivated by the dynamic convolution \cite{De2016DynamicFilter,Su2019Pixel,Han2021DynamicFilterSurvey} as well as the data-driven discretizations \cite{Yohai2019Learningdata-drivenDiscretizations} in 1D cases, we employ the \textit{Trainable Dynamic Difference Layer} (TDDL) to generate localized $K_\theta$. In the TDDL, the parameter $\theta$ varies among grid points to capture local features, which is realized via
\begin{equation}
    K_\theta(s,t;\bm U,k,l)=\bar K_{\mathcal{H}(G(\bm U);k,l)}(s,t),
\end{equation}
where $\mathcal{H}$ is a hypernetwork \cite{DavidHaHypernetwork}. In this paper, we choose CNNs containing 3 convolutional layers with ReLU for the implementation of $\mathcal{H}$. A schematic diagram for the module is displayed in Fig.\ref{fig:method_pdenet_plus}(b). Suppose that the TDDL module is fed with a 2D field $\bm U_j=(V_1,\cdots,V_c)\in\mR^{c\times h\times w}$ with $c$ channels which we need to take derivatives on. We take the $(p,q)$ derivative of the first feature $V_1$ with a truncation order $(r+1)$ as an example. First, a hypernetwork $\mathcal{H}_{pq}$ receiving $\bm U_j$ gives
\[\bm W_{pq}=\mathcal{H}_{pq}(\bm U_j)\in\mR^{m\times h\times w},\]
where the number of output channels $m$ equals the dimension of parameter space $\Theta$. The output features $\bm W_{pq}(\cdot,k,l)\in\mR^m$ for each grid point $(k,l)$ are then reshaped into a local moment matrix $M_{pq}(\cdot,\cdot;k,l)\in\mR^{(2L+1)\times(2L+1)}$ with the upper-left constants initialized before the training stage. Afterwards, the local moment matrix for each grid point is converted into the corresponding local kernel $K_{pq}(\cdot,\cdot;k,l)\in\mR^{(2L+1)\times(2L+1)}$, followed by a point-wise convolutional operation with the feature $V_1$
\[D_\theta(G(V_1))(k,l)=\sum_{s,t=-L}^LK_{pq}(s,t;k,l)V_1(k+s,l+t)\]
for each grid point $(k,l)$. Other combinations of the derivatives and the features can be calculated in this way simultaneously, and all the results are then gathered and sent to the finite difference bank.

\section{Numerical experiments}\label{sec:experiments}
We select three different types of PDEs, namely, the viscous Burgers' equation, the FitzHugh-Nagumo reaction-diffusion equation as well as the Navier-Stokes equation to test our proposed PDE-Net++ architecture. Comprehensive comparisons between the performances of PDE-Net++ and recent well-developed deep-learning models are illustrated and discussed.
Additional experimental settings and some extended experimental results are included in the Appendix \ref{sec:supplementary-for-experiments}.

\subsection{Training settings}
\paragraph{Data generation}
For each PDE setting, a high-resolution numerical solver with full knowledge of the operator $\mL$ in Eq. (\ref{eq:def_a}) is employed to generate the ground truth with a small time step $\delta t$ on a finer $256\times256$ grid, which is then downsampled to a coarser $64\times64$ grid with a much larger time step $\Delta t$ constituting both the training datasets and the testing datasets. The number of trajectories and time steps are denoted by $N$, $N'$ and $M$, $M'$ for the training and the testing datasets, respectively. The initial conditions of the trajectories in the datasets are sampled with the same distributions, but the simulation time in the testing dataset is longer than those in the training dataset in order to analysis the extrapolation performances of the models. Additionally, the training dataset is mixed with noise since the observation data are usually noisy in the real world. For notational convenience, the ground truth, the noisy data, and the predictions are denoted by $\bm U$, $\tilde{\bm U}$, and $\hat{\bm U}$, respectively, where
\[\tilde{\bm U}=\bm U+0.001\hat\sigma\varepsilon.\]
$\hat\sigma$ represents the standard deviation of $\bm U$ for the time step, and $\varepsilon$ is a field of white noise sampled from the standard Gaussian distribution.
\paragraph{Configurations of PDE-Net++}
The difference operators in our PDE-Net++ architecture are implemented in four different ways to approximate the associated derivatives. (I) \textbf{FDM}, the numerical difference schemes same as what we use for data generation (without any learnable parameters); (II) \textbf{Moment}, the convolutional difference layers with moment-constrained kernels proposed in PDE-Net \cite{long2018pde}; (III) \textbf{TFDL} and (IV) \textbf{TDDL} first proposed in this paper. As for the backbones $\mathcal{F}_{\textrm{NN}}$ for the PDE-Net++ models, we choose (i) \textbf{U-Net} \cite{OlafRonneberger2015UNetCN} coming from the field of computer vision; (ii) \textbf{ConvResNet} used in \cite{liu2022predicting}; (iii) \textbf{FNO} \cite{li2020fourier} together with its variant (iv) \textbf{F-FNO} \cite{tran2021factorized}, and (v) \textbf{Galerkin Tansformer} \cite{ShuhaoCao2021ChooseAT} widely used in operator learning tasks.

\paragraph{Black-Box models}
In contrast, we substitute the dynamical term $\mL(\bm x,\bm U_j)$ in Eq. (\refeq{eq:forward-euler-scheme}) as a whole with a neural network implemented by a certain type of backbone mentioned above directly, which is exactly the same idea as those in \cite{TobiasPfaff2020LearningMS} and thus considered as our baselines.

\paragraph{Metrics}
The measurements of the model performances mainly depend on the \textit{average relative} $L_2$ \textit{error} (abbreviated as $L_2$ \textit{error}), which is calculated with the testing dataset $\left\{\tilde{\bm U}_j^{(i)}\right\}_{j=0,i=1}^{M',N'}$ as
\begin{equation}
    \begin{aligned}
        L_2\ \textit{error} = \frac{1}{N'M'}\sum_{i=1}^{N'}\sum_{j=1}^{M'}R\left(\hat{\bm U}_{j}^{(i)},\bm U_{j}^{(i)}\right).
    \end{aligned}
\end{equation}
Here, the function $R(\cdot,\cdot)$ measuring the relative $L_2$ distance is defined as
\[R(\bm X,\bm Y):=\frac{\|\bm X-\bm Y\|_2}{\|\bm Y\|_2}.\]
Meanwhile, since some models may face the instability problem, we use the \textit{Success Rate} (abbreviated as \textit{SR})
\begin{equation}
    \begin{aligned}
        \textit{SR} = \frac{N'_{\textrm{unstable}}}{N'}\times100\%
    \end{aligned}
\end{equation}
as another measurement for prediction stability, where $N'_\textrm{unstable}$ is the number of failed simulations (with $L_2$ \textit{error}s exceeding $1.0$) during inference.

\begin{table*}
\centering
\caption{Comparison of different methods under different backbones for the three types of PDEs.}
\label{tb:different_method}
\renewcommand{\arraystretch}{1.2} 
{\small
\setlength{\tabcolsep}{1.5mm}{ 
    \begin{tabular}{c | c | c | c  c | c  c | c  c }
        \hline
        \multirow{2}{*}{\textbf{Backbone}} & \multirow{2}{*}{\textbf{Method}} & \multirow{2}{*}{\textbf{Derivatives}} & \multicolumn{2}{c|}{\textbf{Burgers}} & \multicolumn{2}{c|}{\textbf{FN}} & \multicolumn{2}{c}{\textbf{NS}} \\
        \cline{4-9}
        & & & \textbf{\textit{SR}} & \boldmath{$L_2$} \textbf{\textit{error}} & \textbf{\textit{SR}} & \boldmath{$L_2$} \textbf{\textit{error}} & \textbf{\textit{SR}} & \boldmath{$L_2$} \textbf{\textit{error}} \\
        \hline
        \multirow{5}{*}{\makecell[c]{U-Net \\ \cite{OlafRonneberger2015UNetCN}}} & Black-Box & - & 100\% & 9.718e-2 & 100\% & 2.103e-2 & 100\% & 5.276e-3\\
        \cline{2-9}
        & \multirow{4}{*}{PDE-Net++} & FDM & 96\% & 7.624e-3 & 100\% & 5.125e-4 & 98\% & 8.029e-3 \\
        & & Moment & 100\% & 4.210e-3 & 100\% & 3.373e-4 & 100\% & 6.865e-3 \\
        & & TFDL & 100\% & 4.970e-3 & - & - & 100\% & 7.153e-3 \\
        & & TDDL &  100\% & 3.114e-3 & 100\% & 2.080e-4 & 100\% & 2.748e-3 \\
        \hline
        \multirow{6}{*}{\makecell[c]{ConvResNet \\ \cite{liu2022predicting}}} & Black-Box & - & 100\% & 1.050e-1 & 100\% & 1.852e-2 & 100\% & 1.114e-1 \\
        \cline{2-9}
        & \multirow{4}{*}{PDE-Net++} & FDM & 96\% & 1.204e-2 & 100\% & 1.893e-4 & 61\% & 1.915e-3 \\
        & & Moment & 97\% & 6.151e-3 & 100\% & 1.232e-4 & 100\% & 1.819e-3 \\
        & & TFDL & 100\% & 7.767e-3 & - & - & 100\% & 1.936e-3 \\
        & & TDDL & 100\% & 4.481e-3 & 100\% & 1.023e-4 & 100\% & 1.951e-3\\
        \hline
        \multirow{5}{*}{\makecell[c]{FNO \\ \cite{li2020fourier}}} & Black-Box & - & 100\% & 1.162e-2 & 100\% & 2.586e-4 & 100\% & 3.204e-3\\
        \cline{2-9}
        & \multirow{4}{*}{PDE-Net++} & FDM & 98\% & 1.036e-2 & 100\% & 8.200e-5 & 36\% & 4.710e-3 \\
        & & Moment & 99\% & 4.345e-3 & 100\% & 6.475e-5 & 100\% & 7.471e-4 \\
        & & TFDL & 100\% & 4.664e-3 & - & - & 100\% & 7.971e-4 \\
        & & TDDL & 100\% & 3.034e-3 & 100\% & 7.118e-5 & 100\% & 9.081e-4 \\
        \hline
        \multirow{5}{*}{\makecell[c]{F-FNO \\ \cite{tran2021factorized}}} & Black-Box & - & 100\% & 2.064e-2 & 100\% & 5.489e-3 & 100\% & 2.721e-3 \\
        \cline{2-9}
        & \multirow{4}{*}{PDE-Net++} & FDM & 96\% & 9.545e-3 & 100\% & 6.305e-5 & 23\% & 3.331e-3 \\
        & & Moment & 100\% & 4.529e-3 & 100\% & 6.357e-5 & 100\% & 1.265e-3 \\
        & & TFDL & 100\% & 4.456e-3 & - & - & 100\% & 1.291e-3 \\
        & & TDDL & 100\% & 2.920e-3 & 100\% & 6.230e-5 & 100\% & 1.333e-3 \\
        \hline
        \multirow{5}{*}{\makecell[c]{Galerkin \\ Transformer \\ \cite{ShuhaoCao2021ChooseAT}}} & Black-Box & - & 100\% & 8.919e-2 & 100\% & 1.708e-2 & 100\% & 2.821e-2 \\
        \cline{2-9}
        & \multirow{4}{*}{PDE-Net++} & FDM & 96\% & 9.316e-3 & 100\% & 1.133e-3 & 94\% & 1.802e-3 \\
        & & Moment & 100\% & 5.473e-3 & 100\% & 4.101e-4 & 99\% & 7.173e-3 \\
        & & TFDL & 100\% & 5.307e-3 & - & - & 100\% & 8.872e-3 \\
        & & TDDL & 100\% & 5.755e-3 & 100\% & 1.249e-3 & 100\% & 1.041e-2 \\
        \hline
        \end{tabular}
}
}
\end{table*}

\begin{figure*}
    \centering
    \subfigure[Burgers' equation]{
        \label{fig:burgers_l2_error_with_timestep}
        \includegraphics[width=0.3\textwidth]{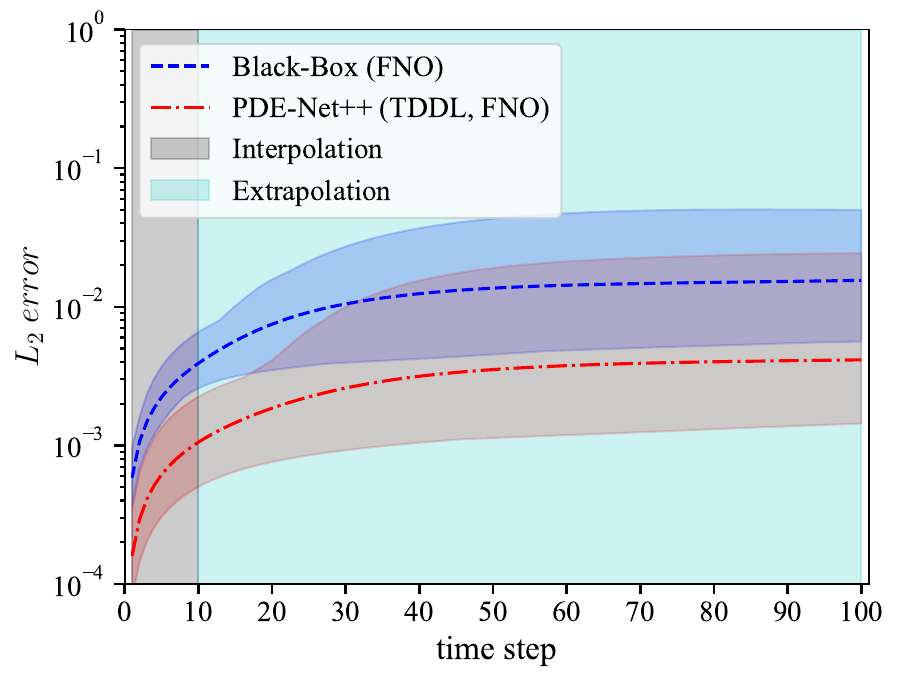}
    }
    \subfigure[FN equation]{
        \label{fig:fn_l2_error_with_timestep}
        \includegraphics[width=0.3\textwidth]{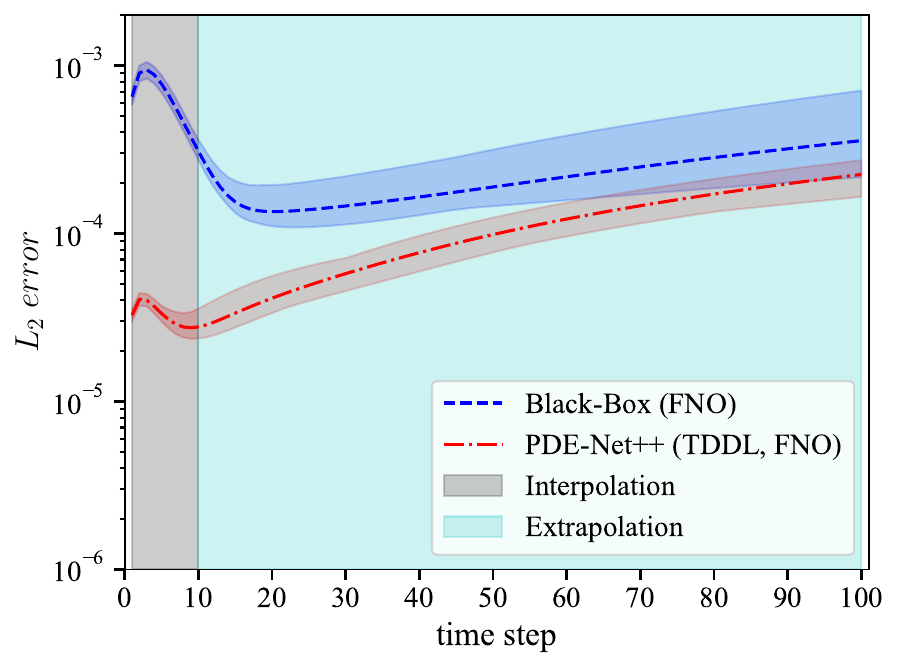}
    }
    \subfigure[NS equation]{
        \label{fig:ns_l2_error_with_timestep}
        \includegraphics[width=0.3\textwidth]{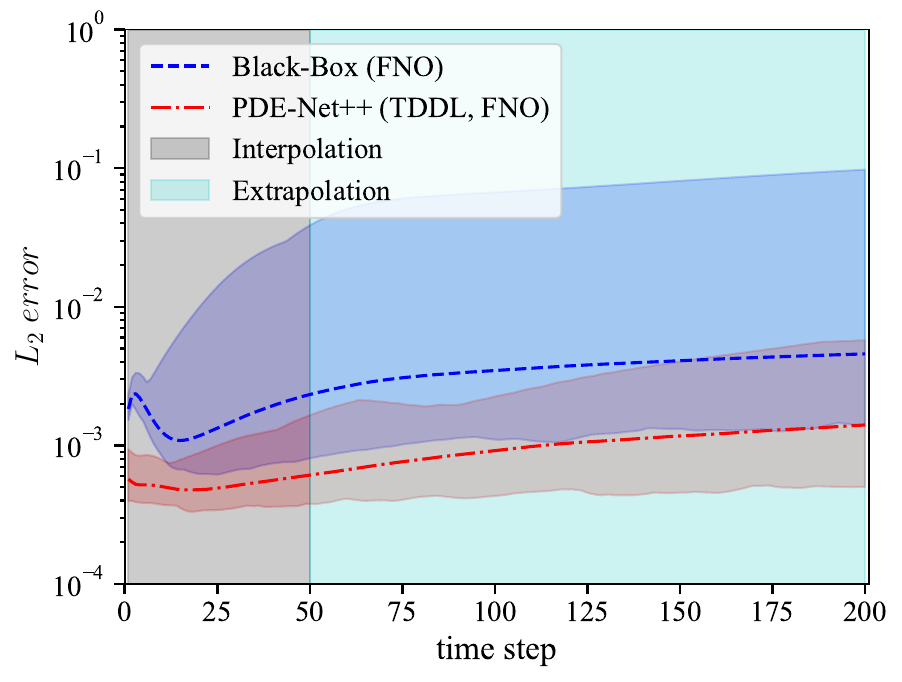}
    }
    \caption{
        $L_2$ \textit{error}s at each time step. The shaded area shows the maximum and minimum $L_2$ \textit{error}s corresponding to all testing samples, and the dashed lines represent the mean.
    }
\end{figure*}

\subsection{Burgers' equation}\label{sec:burgers}
First of all, we study the 2-D viscous Burgers' equation, which is one of the most fundamental non-linear PDEs describing spatio-temporal dynamics. The spatial solution domain is fixed as $\Omega = [0, 2\pi]^2$ with a periodic boundary condition, and we solve for $\bm U=(u(x,y,t),v(x,y,t))^\trans$
\begin{equation}\label{eq:burgers}
    \begin{aligned}
        \frac{\partial\bm U}{\partial t} & = -\bm U \cdot \nabla\bm U + \nu \Delta\bm U +\bm f(x,y,\bm U), \\
        \bm U|_{t=0}                     & = \bm U_0(x,y)
    \end{aligned}
\end{equation}
where the viscosity $\nu = 0.05$. Suppose that both the convection and the diffusion terms are known, but no further information about the forcing term $\bm f$ is provided. The forcing term is set as
\[\bm f(x,y,\bm U) = (\sin(v) \cos(5x+5y), \sin(u) \cos(5x-5y))^\trans\]
when generating the training and testing datasets. The time step $\Delta t = 16\delta t = 0.01$, and the data sizes are set as $N=1000$, $N'=100$ and $M=10$, $M'=100$.

Table \ref{tb:different_method} shows that the PDE-Net++ architecture outperforms the black-box methods under the same backbones, which indicates that the explicit encoding of the known terms in the PDE helps improve the prediction accuracy significantly. Meanwhile, as for the implementations of the difference operators of PDE-Net++, we find that the non-trainable difference operators (FDM) reduce the prediction errors to some extent, but they are faced with the instability issue. We need to emphasize that the trainable difference layers, especially the proposed TFDL and TDDL modules are capable of resolving the instability brought by the difference schemes.
In addition, as illustrated in Figure \ref{fig:burgers_l2_error_with_timestep}, PDE-Net++ produces stable predictions for the $M'=100$ time steps with relatively high accuracy despite the fact that it is only trained with the first $M=10$ steps.

The learned dynamics are shown in Figure \ref{fig:burgers_result}, including the approximated known part $\hat{\bm\Phi}$ produced by the difference operators and the outputs of the FNO backbone designed to account for the unknown part $\bm f$. The predictions of $\bm U$ are nearly indistinguishable from the reference, but it can be seen from the recovered unknown part $\bm f$ that the non-trainable difference operator (FDM) is unable to express the derivatives with high accuracy where the known part $\bm\Phi$ changes dramatically. The poor performances of the non-trainable difference operator (FDM) should result from the failure of the CFL conditions as our time step $\Delta t$ is quite large. Meanwhile, the shift from the RK method used for data generation to the simple Euler forward integration may lead to instability. On the contrary, the trainable difference operator (TDDL) does not suffer from such issue, which implies the flexibility of the trainable difference operators help search for a better collection of coefficients for the specific time step combined with the integration method.

\begin{figure}
    \centering
    \includegraphics[width=0.95\columnwidth]{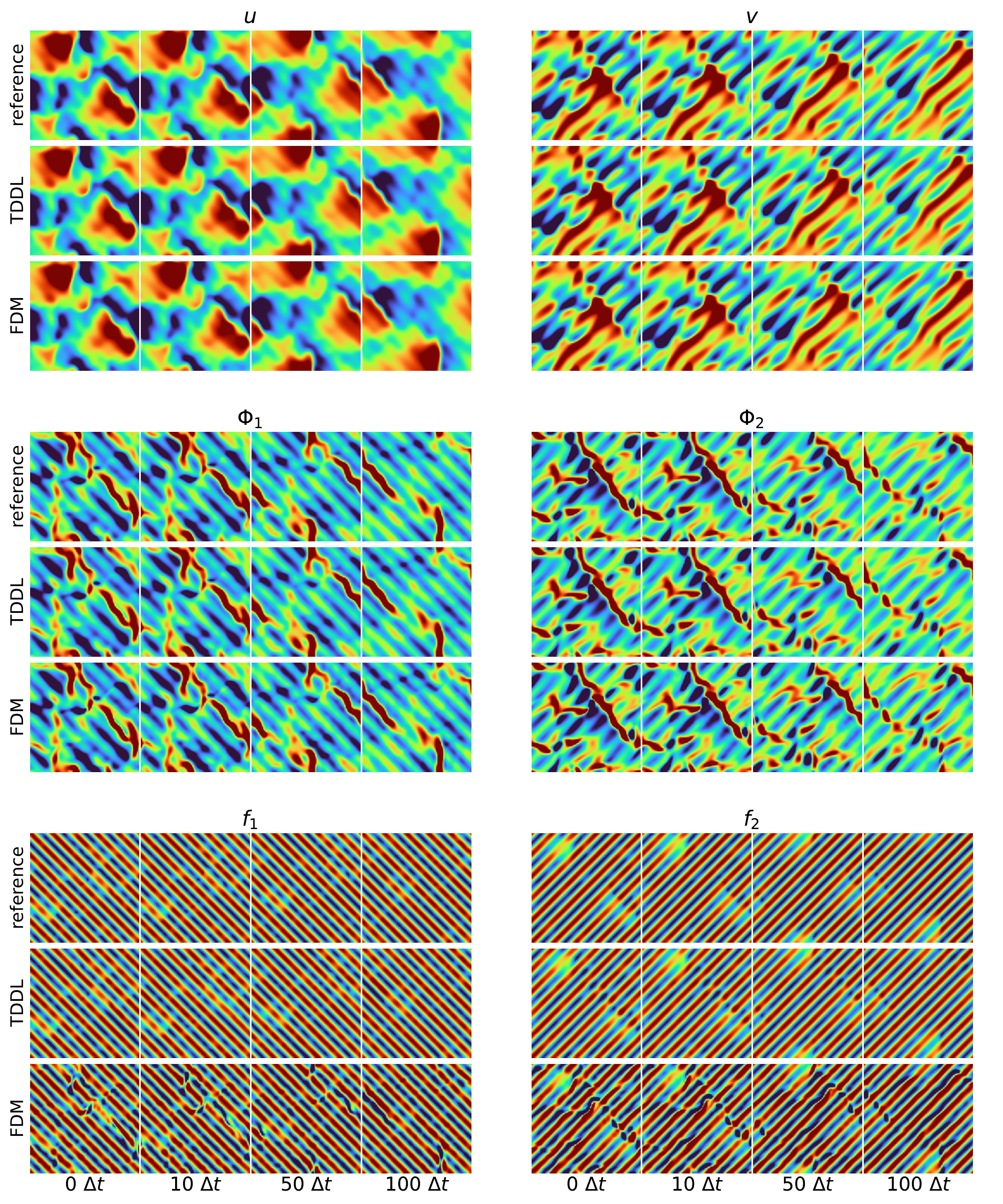}
    \caption{\textbf{Burgers' equation:} Comparison of different choices of the difference operators for the PDE-Net++ architecture with the FNO backbone. The ground truth is used for reference. $\bm U=(u,v)^\trans$ is the solution. $\bm \Phi=(\Phi_1,\Phi_2)^\trans$ and $\bm f=(f_1,f_2)^\trans$ represents the learned known part and the remaining unknown part, respectively.
    }
    \label{fig:burgers_result}
\end{figure}

\begin{figure}
    \begin{center}
        \centerline{\includegraphics[width=0.9\columnwidth]{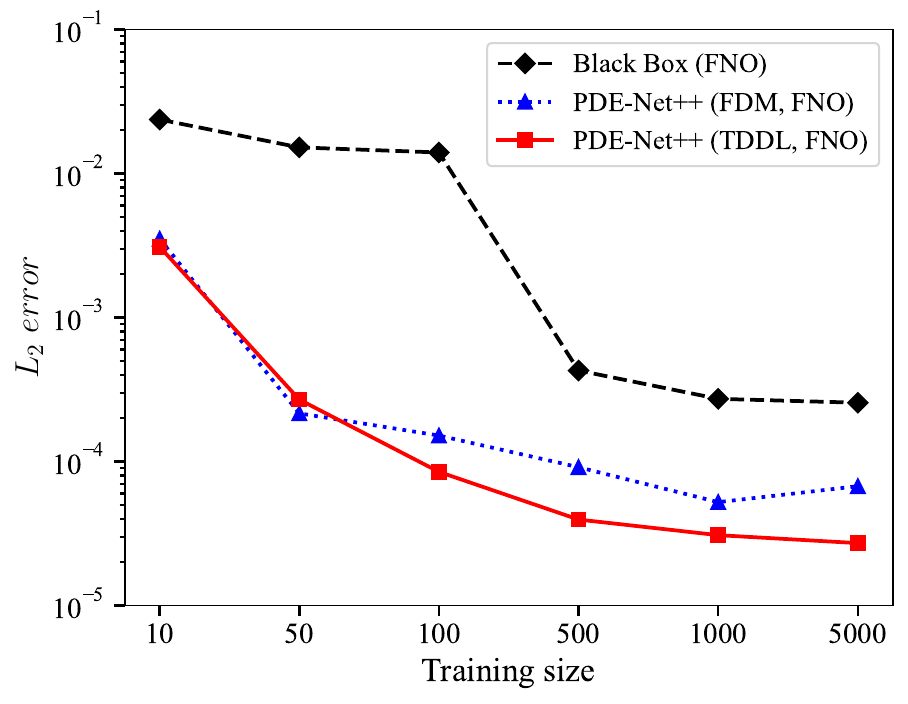}}
        \caption{\textbf{FN equation:} Ablation for the training sizes.}
        \label{fig:fn_ablation_train_size}
    \end{center}
\end{figure}

\subsection{FitzHugh-Nagumo equation}\label{sec:fn}
For the second case, we consider the 2-D FitzHugh-Nagumo (FN) model as a typical example of the reaction-diffusion dynamics. The FN equation is defined as
\begin{equation}\label{eq:fn_rd}
    \begin{aligned}
        \frac{\partial\bm U}{\partial t} & = \gamma \Delta\bm U +\bm R(\bm U), \\
        \bm U|_{t=0}                     & =\bm U_0(x,y)
    \end{aligned}
\end{equation}
with $\bm U=(u(x,y,t),v(x,y,t))^\trans$, where the diffusion coefficient $\gamma=1.0$. Here, the reaction term
\[\bm R(\bm U)=(u-u^3-v+\alpha,\beta(u-v))^\trans,\,\alpha=0.01,\,\beta=0.25\]
is assumed as the unknown part when we establish the PDE-Net++ models. The domain $\Omega$ is set as $[0,6.4]^2$ with the periodic boundary conditions.

To generate our datasets, the data sizes remain the same as those of Burgers' equation, but we modify the time steps to $\Delta t=200\delta t=0.002$ for the purpose of enabling the simulated trajectories to capture more details of the dynamics.

The comparison of different methods is shown in Table \ref{tb:different_method}. Note that the TFDL module is excluded in that it is specially defined for the first-order derivatives. The $L_2$ \textit{error}s still show the superiority of our PDE-Net++ architectures over the black-box models, and the PDE-Net++ architectures equipped with trainable difference operators (Moment, TDDL) generally predict more accurately than those equipped with non-trainable ones (FDM). No failed cases are found in these experiments, and our PDE-Net++ architecture with the proposed TDDL module and the FNO backbone achieves the minimum prediction error.

Figure \ref{fig:fn_ablation_train_size} compares the performances of the black-box method and the PDE-Net architectures with FDM or TDDL as the difference operators when the sizes of training dataset vary. It can be concluded that the PDE-Net++ architecture is able to learn the hidden dynamics $\mL_{\textrm{unknown}}$ from merely 100 training trajectories, while the black-box model is not well trained unless the training size is increased to 1000.

\subsection{Navier-Stokes equation}\label{sec:ns}
Finally, we test our models with the 2-D Navier-Stokes (NS) equation as a much harder case. The equation writes in vorticity form
\begin{equation}\label{eq:ns}
    \begin{aligned}
        \frac{\partial w}{\partial t} & = -\bm U \cdot \nabla w + \nu \Delta w + f(x,y), \\
        \nabla \cdot \bm U            & = 0,                                             \\
        w|_{t=0}                      & = w_0(x,y).
    \end{aligned}
\end{equation}
with the periodic boundary condition for the spatial solution domain $\Omega=[0,1]^2$. $\bm U=(u(x,y,t),v(x,y,t))^\trans$ is the velocity field, and $w=\nabla\times\bm U=\frac{\partial v}{\partial x}-\frac{\partial u}{\partial y}$ is the vorticity we need to solve. The viscosity coefficient $\nu$ is set as 0.001. For a smaller viscosity coefficient when the dynamics becomes more chaotic, we explore the performances of the models in the Appendix \ref{appendix:ns}.

The initial vorticity $w_0(x,y)$ is randomly sampled from $\mathcal{N}(0, 25 (-\Delta + 25 I)^{-3})$ as a Gaussian random field \cite{YangLiu2019AdvancesIG} and then z-score normalized. The forcing function $f(x,y)$ is generated in the same way, but we have to stress that all the trajectories starting from different initial conditions in the datasets share the same $f(x,y)$. For each time step, PDE-Net++ recovers the velocity field $\bm U=(u,v)=\left(-\frac{\partial\psi}{\partial y},\frac{\partial\psi}{\partial x}\right)$ after obtaining the stream function $\psi=\Delta^{-1}w$ via (discrete) Fourier transform on $w$. Then the gradient $\nabla w$ and the diffusion $\Delta w$ are calculated the same way as before. The time step $\Delta t=500\delta t=0.025$, and the data sizes are $N=1000$, $N'=100$ and $M=50$, $M'=200$.

As depicted in Table \ref{tb:different_method}, with the same backbone, PDE-Net++ attains a higher accuracy compared with the black-box models. Meanwhile, we find that PDE-Net++ with non-trainable difference operators (FDM) can hardly keep stable for the NS equation case, which may attribute to the fact that the dynamic system is much more complicated than the previous two cases. Although the TFDL and the TDDL modules predict almost as accurately as the moment-constrained convolution (Moment), the success rate reveals the consistent stability of our proposed modules.

\section{Discussion and future work}
PDE-Net++ is a newly proposed hybrid neural network architecture that incorporates physics prior with black-box models.
Comparisons of accuracy and stability have exhibited that even partly embedding the governing equations into the network can significantly improve the accuracy of predictions. Besides, inspired by the trainable difference operators in the existing works, two additional modules named TFDL and TDDL are firstly introduced. All the experiment results have shown that these two modules are able to stablize the roll-out process effectively.

Current architecture can be applied to any other prediction tasks for spatio-temporal dynamics as long as partial knowledge of the underlying PDEs is provided, although we have to emphasize that necessary modifications for the parameterizations of the feasible schemes are needed. For instance, in deep-learning based weather prediction tasks, till now, only black-box models such as U-Net \cite{Weyn2021modifiedDLWP}, ResNet \cite{Rasp2021ResNet,Clare2021ResNetProbability}, GNN \cite{keisler2022GNNs,lam2022graphcast}, and Transformer \cite{Guibas2021AFNO,pathak2022fourcastnet,Liu2021SwinTransformer,bi2022pangu} have been tested. Since the dynamical core has been studied for decades by experts, with the aid of PDE-Net++, existing state-of-the-art models are believed to be able to achieve better performance.

\bibliographystyle{icml2023}
\bibliography{references}

\newpage
\appendix
\onecolumn
\section{Mathematical backgrounds}
Some of the following conclusions can be found in \cite{long2018pde,So2021DSN}. We restate and summarize the simplified versions for readers' convenience.
\begin{Prop}\label{prop:moment-constraints}
    For any smooth function $h$ and a compactly supported kernel $K$,
    \[(K\circledast h)(x,y)=\frac{\partial^{ p+ q}}{\partial x^{ p}\partial y^{ q}}h(x,y)+\mathcal{O}\left(\left((\Delta x)^2+(\Delta y)^2\right)^{(r+1)/2}\right)\]
    if and only if Eq. (\ref{eq:moment-constraints}) holds.
\end{Prop}
\begin{proof}
    Starting from Eq. (\ref{eq:conv-difflayer-concept}), as $\Delta x$ and $\Delta y$ approach zero, we have by truncating the Taylor series that
    \begin{align*}
        (K\circledast h)(x,y) & =\sum_{s,t\in\mZ}K(s,t)h(x+s\Delta x,y+t\Delta y)                                                                                                                                                                         \\
                              & =\sum_{s,t\in\mZ}\left[K(s,t)\sum_{l=0}^r\frac{(s(\Delta x)\partial_x+t(\Delta y)\partial_y)^l}{l!}h(x,y)+\mathcal{O}\left(\left((s\Delta x)^2+(t\Delta y)^2\right)^{(r+1)/2}\right)\right]                               \\
                              & =\sum_{s,t\in\mZ}K(s,t)\sum_{l=0}^r\sum_{u+v=l}\frac{(s\Delta x)^u}{u!}\frac{(t\Delta y)^v}{v!}\frac{\partial^lh}{\partial x^u\partial y^v}(x,y)+\mathcal{O}\left(\left((\Delta x)^2+(\Delta y)^2\right)^{(r+1)/2}\right) \\
                              & =\sum_{u+v\le r}M(K)(u,v)\frac{\partial^lh}{\partial x^u\partial y^v}(x,y)+\mathcal{O}\left(\left((\Delta x)^2+(\Delta y)^2\right)^{(r+1)/2}\right),
    \end{align*}
    and that the proposition immediately follows.
\end{proof}
Note that we take additional spatial mesh steps $\Delta x$ and $\Delta y$ together with the truncation order $(r+1)$ into consideration compared with the conclusion given in PDE-Net. From now on, we fix the sizes of both the kernel and the moment matrix as $(2L+1)\times(2L+1)$.
\begin{Prop}\label{prop:Q-invertible}
    There exists a unique moment matrix $M(K)$ corresponding with any kernel $K$ and vice versa. In other words, the map $K\mapsto M(K)$ is invertible.
\end{Prop}
\begin{proof}
    By identifying the kernel $K$ with a $(2L+1)\times(2L+1)$ matrix with indices from $-L$ to $L$ inclusively, Eq. (\ref{eq:moment-constraints}) implies
    \[M(K)(u,v)=\sum_{s=-L}^L\sum_{t=-L}^L\frac{s^u(\Delta x)^u}{u!}K(s,t)\frac{t^v(\Delta y)^v}{v!}.\]
    Hence $M(K)=Q_xKQ_y^\trans $ for
    \[Q_x(u,s)=\frac{s^u(\Delta x)^u}{u!},\,Q_y(u,s)=\frac{s^u(\Delta y)^u}{u!},\]
    where $u=0,\cdots, 2L$ and $s=-L,\cdots,0,\cdots,L$. $0^0$ is defined as $1$ as convention.
    \begin{align*}
        Q_x=
        \begin{pmatrix}
            0!                     \\
             & 1!                  \\
             &    & \ddots         \\
             &    &        & (2L)!
        \end{pmatrix}^{-1}
        \begin{pmatrix}
            1                 & \cdots & 1      & \cdots & 1                \\
            (-L\Delta x)^1    & \cdots & 0^1    & \cdots & (L\Delta x)^1    \\
            \vdots            &        & \vdots &        & \vdots           \\
            (-L\Delta x)^{2L} & \cdots & 0^{2L} & \cdots & (L\Delta x)^{2L}
        \end{pmatrix}
    \end{align*}
    is a multiplication of an invertible diagonal matrix and a non-singular Vandermonde matrix and thus invertible. Analogously, $Q_y$ is invertible as well, and it follows that the linear transformation $K\mapsto M(K)$ is invertible.
\end{proof}
\begin{Prop}\label{prop:kernels-in-hyperplane}
    For fixed $(p,q)$ and $r$ such that $p+q+r\le 2L$, the set $\mathcal{S}_{(p,q)}^r$ consisting of all the kernels satisfying Eq. (\ref{eq:moment-constraints}) forms a hyperplane of dimension $(2L+1)^2-(p+q+r+1)(p+q+r+2)/2$.
\end{Prop}
\begin{proof}
    Let $E_{ij}$ be the matrix with the only non-zero element $1$ at the position $(i,j)$. By the definition of $\mathcal{S}_{(p,q)}^r$,
    \begin{align*}
        K\in \mathcal{S}_{(p,q)}^r & \Leftrightarrow \exists c_{uv}\in\mR,\,M(K)=E_{pq}+\sum_{u+v>p+q+r}c_{uv}E_{uv}                                           \\
                                   & \Leftrightarrow \exists c_{uv}\in\mR,\,Q_xKQ_y^\trans =E_{pq}+\sum_{u+v>p+q+r}c_{uv}E_{uv}                                \\
                                   & \Leftrightarrow \exists c_{uv}\in\mR,\,K=Q_x^{-1}E_{pq}Q_y^{-\trans} +\sum_{u+v>p+q+r}c_{uv}Q_x^{-1}E_{uv}Q_y^{-\trans} .
    \end{align*}
    Here we abbreviate $\left(Q_y^\trans \right)^{-1}$ as $Q_y^{-\trans} $. Define
    \[K_{uv}^0=Q_x^{-1}E_{uv}Q_y^{-\trans} =Q_x^{-1}[:,i]Q_y^{-1}[:,j]^\trans \]for any $u,v=0,\cdots, 2L$, where $Q_x^{-1}[:,i]$ stands for the $i$-th column of $Q_x^{-1}$, then
    \[\mathcal{S}_{(p,q)}^r=K_{pq}^0+\Span\left(K_{uv}^0\right)_{u+v>p+q+r},\]
    where the notation $\Span(\cdot)$ represents the subspace spanned by the elements inside.
    \begin{equation}\label{eq:channel-num}
        \begin{aligned}
            \sum_{u+v>p+q+r}1 & =(2L+1)^2-\sum_{u+v\le p+q+r}1=(2L+1)^2-\frac12(p+q+r+1)(p+q+r+2)
        \end{aligned}
    \end{equation}
\end{proof}
\begin{Rem}
    We have to emphasize that all the kernels $\left\{K_{uv}^0\right\}_{u,v=0}^{2L}$ mentioned in the proof of Prop. \ref{prop:kernels-in-hyperplane} are constants and independent of the order $r$. Consequently, we prepare and save $\left\{K_{uv}^0\right\}_{u,v=0}^{2L}$ according to Prop. \ref{prop:Q-invertible} in advance before the training stage, which can be shared among convolutional layers corresponding to different derivatives afterward.
\end{Rem}

\begin{Prop}[Flipped kernels]\label{prop:flipped-kernels}
    For $(p,q)=(1,0)$ and any $r$, the kernel $K$ satisfies Eq. (\refeq{eq:moment-constraints}) if and only if
    $K'(s,t)=-K(-s,t)$ satisfies Eq. (\refeq{eq:moment-constraints}).
\end{Prop}
\begin{proof}
    Consider the moment matrices for $K$ and $K'$.
    \begin{align*}
        M(K')(u,v) & =\sum_{s=-L}^L\sum_{t=-L}^L\frac{(s\Delta x)^u(t\Delta y)^v}{u!v!}K'(s,t)                               \\
                   & =-\sum_{s=-L}^L\sum_{t=-L}^L\frac{(s\Delta x)^u(t\Delta y)^v}{u!v!}K(-s,t)                              \\
                   & =(-1)^{u+1}\sum_{s=-L}^L\sum_{t=-L}^L\frac{(-s\Delta x)^u(t\Delta y)^v}{u!v!}K(-s,t)                    \\
                   & =(-1)^{u+1}\sum_{s=-L}^L\sum_{t=-L}^L\frac{(s\Delta x)^u(t\Delta y)^v}{u!v!}K(s,t)=(-1)^{u+1}M(K)(u,v).
    \end{align*}
    Hence $M(K)(u,v)=0$ if and only if $M(K)(u,v)=0$, and $M(K)(1,0)=1$ if and only if $M(K')(1,0)=1$.
\end{proof}

\section{Some examples for numerical schemes encoding local features}\label{sec:schemes-with-local-features}
In this part, we list some types of classical numerical methods accounting for local features. For simplicity, the numerical methods are illustrated in 1D cases since similar derivations can be applied to high-dimensional cases with little difficulty. We choose the forward Euler integration with a fixed time step $\Delta t$, and the spatial grid points are assumed to be distributed uniformly with $\Delta x$ as the interval length. The superscripts and the subscripts indicate the indices along the temporal and spatial dimensions, respectively.
\subsection{Upwind schemes}
Consider the following 1D advection equation
\[\frac{\partial u}{\partial t}+c\frac{\partial u}{\partial x}=0,\,u:[0,1]\to\mR.\]It follows that the forward update rule reads
\[U_j^{m+1}=U_j^m-(c\Delta t)D(U^m)_j,\]
where $U_j^m$ represents the numerical solutions, and $D(U^m)_j$ is an approximation of $\partial_x U^m$ at the spatial index $j$. Intuitively one may choose the central difference $D(U^m)_j=(U^m_{j+1}-U^m_{j-1})/(2\Delta x)$ for a higher order of approximation, but such choice will suffer from unconditional instability by the von Neumann stability analysis \cite{Thomas2013NPDE-FDM}. Upwind schemes adopt a lower-order discretization to improve the stability, which writes
\begin{equation}\label{eq:1st-order-upwind}
    U_j^{m+1}=(1-|\mu|)U_j^m+\frac{\mu+|\mu|}2U_{j-1}^m-\frac{\mu-|\mu|}2U_{j+1}^m
\end{equation}
for a first order discretization and
\begin{equation}\label{eq:2nd-order-upwind}
    U_j^{m+1}=\frac{2-3|\mu|}2U_j^m+(\mu+|\mu|)U_{j-1}^m-(\mu-|\mu|)U_{j+1}^m-\frac{\mu+|\mu|}4U_{j-2}^m+\frac{\mu-|\mu|}4U_{j+2}^m
\end{equation}
for a second-order discretization, where $\mu=c\Delta t/\Delta x$. Other finite difference methods including the Lax-Friedrichs scheme, the Lax-Wendroff scheme, and the Beam-Warming scheme with various dispersion and dissipation \cite{Thomas2013NPDE-FDM} can be written as such updates, where the coefficient of each grid point on the stencil is a function of $\mu$ as well.
\subsection{Flux limiters}
FVMs model $U_j$ as the average of $u$ at the $j$-th cell, and for 1D hyperbolic PDEs with a scalar conservation law
\[\frac{\partial u}{\partial t}+\frac{\partial}{\partial x}f(u)=0,\,u:[0,1]\to\mR,\]
the discrete schemes usually appear as
\[U_j^{m+1}=U_j^m-\frac{\Delta t}{\Delta x}\left[F_{j+1/2}^m-F_{j-1/2}^m\right],\]
where $F_{j\pm1/2}^m$ approximates the flux passing through the right/left edge of the $j$-th cell. By setting $f(u)=cu$ with a constant $c$, the first-order upwind scheme (\ref{eq:1st-order-upwind}) corresponds to the numerical flux
\[F_{j+1/2}^m=\frac{\mu+|\mu|}2U_j^m+\frac{\mu-|\mu|}2U_{j+1}^m.\]
In general, the upwind scheme is not competitive when compared with higher-order schemes such as the Lax-Wendroff scheme for smooth solutions. However, these higher-order schemes tend to generate oscillations near discontinuities due to their dispersive nature \cite{LeVeque2002FVM}. Based on such observations, flux limiters are often used to switch between low-resolution and high-resolution discretizations. Formally, the numerical flux is defined as
\[F_{j+1/2}^m=\phi_{j+1/2}^mF_{H,j+1/2}^m+\left(1-\phi_{j+1/2}^m\right)F_{L,j+1/2}^m,\]
where $F_{H,j+1/2}^m$ and $F_{L,j+1/2}^m$ stand for a high-order flux and a low-order flux, respectively. The flux limiter $\phi_{j+1/2}^m$ depending on the local features of the solutions determines whether the scheme reduces to the lower-order method or not. the total variation diminishing (TVD) property together with the Courant-Friedrichs-Lewy (CFL) condition plays a significant role in constructing a large number of various flux limiters \cite{Roe1986minmod,VanLeer1974,VanLeer1979}.
\subsection{ENO/WENO reconstruction}
Given cell average
\[\bar u_j=\frac1{\Delta x}\int_{I_j}u(\xi)\md\xi\]
of $u$ over the $j$-th interval $I_j$ for all $j$, ENO and WENO schemes \cite{Liu1994WENO,Shu1998ENOandWENO} essentially look for a polynomial approximation $p_j$ of $u$ within $I_j$ by taking polynomial interpolation on the primitive function of $u$. It follows that the values $u_{j\pm1/2}$ are approximated by the evaluations of $p_j$ on the cell edges. Such feasible polynomials are not unique even when the size of stencils is fixed. For instance,
\begin{align*}
    u_{j+1/2} & =-\frac16\bar u_{j-1}+\frac56\bar u_j+\frac13\bar u_{j+1}+\mathcal{O}\left((\Delta x)^3\right),\textrm{ and} \\
    u_{j+1/2} & =\frac13\bar u_j+\frac56\bar u_{j+1}-\frac16\bar u_{j+2}+\mathcal{O}\left((\Delta x)^3\right).
\end{align*}
ENO and WENO schemes select the most suitable polynomials by certain measurements of smoothness in order to reduce the total variation, and one of the major differences is that WENO admits a weighted combination of polynomials coming from different stencils. Combined with certain monotone fluxes, the reconstruction will give good approximations of the fluxes on the cell edges.

\section{Supplementary materials for the experiments}\label{sec:supplementary-for-experiments}
This section formulates the training loss for the experiments and then summarizes extended experimental results for the three PDEs described in Sec. \ref{sec:experiments}.

\subsection{Training loss}
We train both our PDE-Net++ models and the black-box models with a single time step, and then roll out for hundreds of steps in testing. Formally, the updating rules for the observation $\tilde{\bm U}_j^{(i)}$ of the $i$-th trajectory at time step $j$ read
\begin{equation}\label{eq:pdenet_plus_pred}
    \hat{\bm U}_{j+1}^{(i)}=\tilde{\bm U}_j^{(i)}+\hat{\bm\Phi}\left(\bm x, \tilde{\bm U}_j^{(i)}\right)\Delta t+\mathcal{F}_{\mathrm{NN}}\left(\bm x, \tilde{\bm U}_j^{(i)}\right)\Delta t
\end{equation}
for the PDE-Net++ models and
\begin{equation}\label{eq:black_box_pred}
    \begin{aligned}
        \hat{\bm U}_{j+1}^{(i)} & =\tilde{\bm U}_j^{(i)}+\mathcal{F}_{\mathrm{NN}}\left(\bm x, \tilde{\bm U}_j^{(i)}\right)\Delta t
    \end{aligned}
\end{equation}
for the black-box models.

During the training stage, the learnable components are tuned to reduce the distance between the model predictions and the actual observations of the next step. Suppose that we have the training dataset $\left\{\tilde{\bm U}_j^{(i)}\right\}_{j=0,i=1}^{M,N}$, the training loss $L$ is defined as
\[L=L_{\textrm{pred}}+\lambda L_{\textrm{reg}},\]
where the prediction loss $L_{\textrm{pred}}$ is measured as
\[L_{\textrm{pred}}=\frac{1}{NM}\sum_{i=1}^N\sum_{j=1}^MR\left(\hat{\bm U}_j^{(i)},\tilde{\bm U}_{j}^{(i)}\right).\]
The $\lambda L_{\textrm{reg}}$ term is set to zero for all the black-box models and the PDE-Net++ architectures with the difference operators implemented as ``FDM'' described above. Otherwise, $L_{\textrm{reg}}$ penalizes the $L_1$ norm for the lower-right free parameters in the moment matrices with the hyper-parameter $\lambda=0.001$.

\subsection{Extended experimental results}\label{sec:extended_result}

\subsubsection{Trainable parameters}

\begin{table*}
\renewcommand{\arraystretch}{1.2} 
\centering
\caption{Comparison of the number of trainable parameters under different methods and different backbones for different PDEs.
}
\label{tb:num_params}
\setlength{\tabcolsep}{3.0mm}{ 
    \begin{tabular}{c | c | c | c | c | c }
        \hline
        \multirow{2}{*}{\textbf{Backbone}} & \multirow{2}{*}{\textbf{Method}} & \multirow{2}{*}{\textbf{Derivatives}} & \multicolumn{3}{c}{\textbf{\# Params}} \\
        \cline{4-6}
        & & & \textbf{Burgers} & \textbf{\quad FN \quad} & \textbf{\quad NS \quad} \\
        \hline
        \multirow{5}{*}{\makecell[c]{U-Net \\ \cite{OlafRonneberger2015UNetCN}}} & Black-Box & - & 13395970 & 13395970 & 13395329 \\
        \cline{2-6}
        & \multirow{4}{*}{PDE-Net++} & FDM & 13395970 & 13395970 & 13395329 \\
        & & Moment & 13396134 & 13396046 & 13395411\\
        & & TFDL & 13396134 & - & 13395411 \\
        & & TDDL & 13519590 & 13455350 & 13455539 \\
        \hline
        \multirow{5}{*}{\makecell[c]{ConvResNet \\ \cite{liu2022predicting}}} & Black-Box & - & 509334 & 509334 & 507531\\
        \cline{2-6}
        & \multirow{4}{*}{PDE-Net++} & FDM & 509334 & 509334 & 507531 \\
        & & Moment & 509498 & 509410 & 507613 \\
        & & TFDL & 509498 & - & 507613 \\
        & & TDDL & 632954 & 568714 & 567741 \\
        \hline
        \multirow{5}{*}{\makecell[c]{FNO \\ \cite{li2020fourier}}} & Black-Box & - & 465526 & 465526 & 465377 \\
        \cline{2-6}
        & \multirow{4}{*}{PDE-Net++} & FDM & 465526 & 465526 & 465377 \\
        & & Moment & 465690 & 465602 & 465459\\
        & & TFDL & 465690 & - & 465459 \\
        & & TDDL & 589146 & 524906 & 525587 \\
        \hline
        \multirow{5}{*}{\makecell[c]{F-FNO \\ \cite{tran2021factorized}}} & Black-Box & - & 26086 & 26086 & 25937 \\
        \cline{2-6}
        & \multirow{4}{*}{PDE-Net++} & FDM & 26086 & 26086 & 25937\\
        & & Moment & 26250 & 26162 & 26019 \\
        & & TFDL & 26250 & - & 26019 \\
        & & TDDL & 149706 & 85466 & 86147 \\
        \hline
        \multirow{5}{*}{\makecell[c]{Galerkin \\ Transformer \\ \cite{ShuhaoCao2021ChooseAT}}} & Black-Box & - & 1955102 & 1955102 & 1953821 \\
        \cline{2-6}
        & \multirow{4}{*}{PDE-Net++} & FDM & 1955102 & 1955102 & 1953821\\
        & & Moment & 1955266 & 1955178 & 1953903 \\
        & & TFDL & 1955266 & - & 1953903 \\
        & & TDDL & 2078722 & 2014482 & 2014031\\
        \hline
        \end{tabular}
}
\end{table*}

Compared with the black-box models, PDE-Net++ may have additional trainable parameters lying in the trainable difference operators (if exist). Table \ref{tb:num_params} lists the numbers of trainable parameters for all the experiments mentioned in Sec. \ref{sec:experiments}. With the exception of the F-FNO backbone, the increments resulting from the trainable difference operators (Moment, TFDL, and TDDL) are negligible with the same backbones.

\subsubsection{Burgers' equation}
We also investigate the effect of the number of training samples (training size) on different methods for Burgers' equation. Figure \ref{fig:burgers_ablation_train_size} shows that as the training size increases, PDE-Net++ with the TDDL module consistently outperforms both the other models.

\begin{figure}
    \begin{center}
        \centerline{\includegraphics[width=0.6\columnwidth]{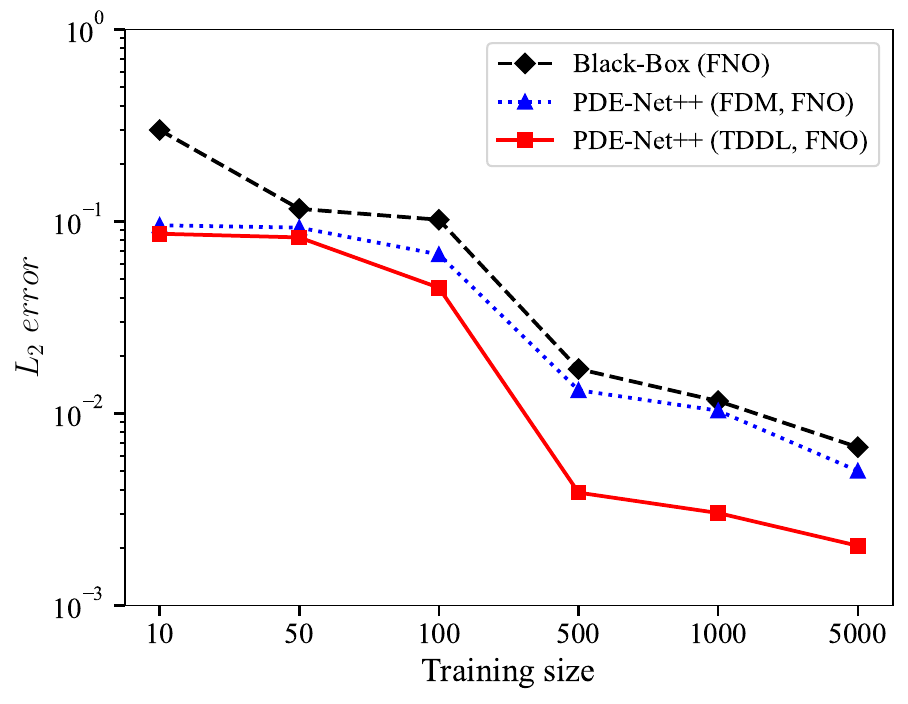}}
        \caption{\textbf{Burgers' equation:} Comparison of PDE-Net++ and the black-box model under various training sizes with the FNO backbone.}
        \label{fig:burgers_ablation_train_size}
    \end{center}
\end{figure}

\subsubsection{FitzHugh-Nagumo equation}
For the FN equation, Figure \ref{fig:fn_result} exhibits snapshots of the reference and predicted solutions of PDE-Net++ (TDDL, FNO) at different time steps during testing. It turns out that our method is able to evolve the diffusion appropriately and that the differences with the reference are indistinguishable.

\begin{figure}
    \begin{center}
        \centerline{\includegraphics[width=0.8\columnwidth]{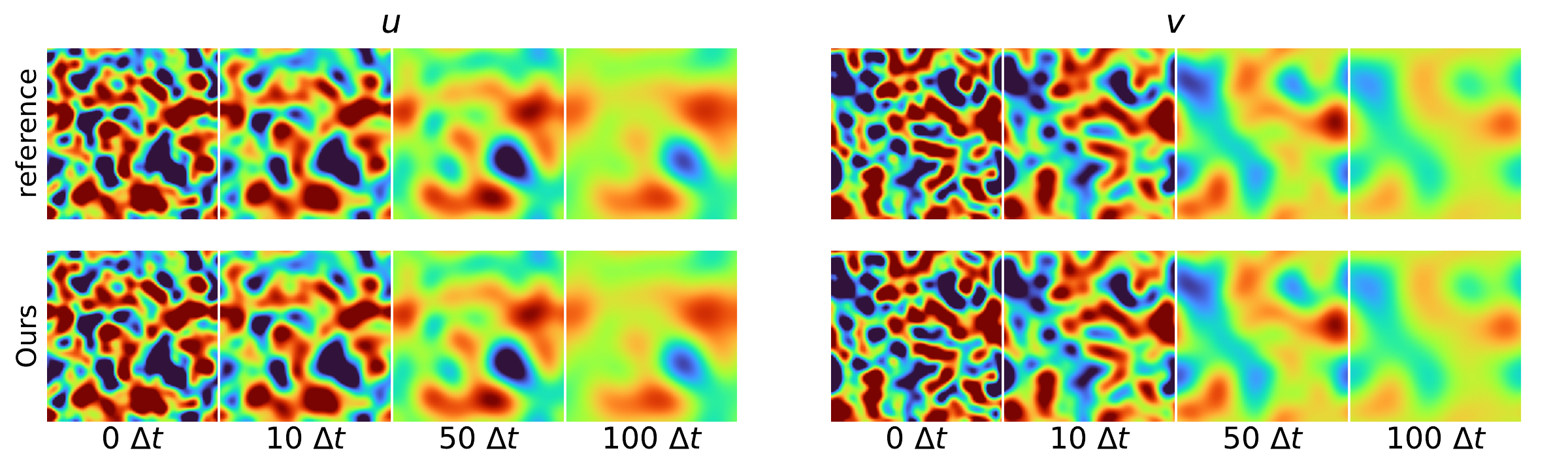}}
        \caption{\textbf{FN equation:} Predicted snapshots at different time steps. `reference' represents ground truth; `Ours' indicates the prediction of PDE-Net++ (TDDL, FNO).
        }
        \label{fig:fn_result}
    \end{center}
\end{figure}

\subsubsection{Navier-Stokes equation}\label{appendix:ns}
The viscosity coefficient $\nu$ has been set to 0.001 in the previous experiments for the NS equation. Here, we also attempt to set $\nu=0.0001$ as a more challenging scenario. The time steps and the number of roll-out steps are changed to $\Delta t=125\delta t=0.00625$ and $M=M'=800$, respectively. The comparison of different methods is displayed in Table \ref{tb:ns_different_method_0.0001}, which indicates that PDE-Net++ (TDDL, FNO) achieved the best performance under such circumstances. In addition, Figure \ref{fig:ns_result} presents a comparison between the predictions of PDE-Net++ (TDDL, FNO) and the ground truth. Four samples are randomly selected from all the testing trajectories of size $N'=100$.

\begin{table}
    \renewcommand{\arraystretch}{1.2} 
    \centering
    \caption{\textbf{Navier-Stokes equation:} Comparison of different methods when viscosity $\nu=0.0001$.}
    \label{tb:ns_different_method_0.0001}
    {\small
        \setlength{\tabcolsep}{1.5mm}{ 
            \begin{tabular}{c | c | c | c  c}
                \hline
                \textbf{Backbone} & \textbf{Method} & \textbf{Derivatives} & \textbf{\textit{SR}}                 & \boldmath{$L_2$} \textbf{\textit{error}}                            \\
                \hline
                \multirow{5}{*}{FNO}               & Black-Box                        & -                                     & 100\%                                & 1.345e-2         \\
                \cline{2-5}
                                                   & \multirow{4}{*}{PDE-Net++}       & FDM                                   & 63\%                                 & 1.079e-2         \\
                                                   &                                  & Moment                                & 97\%                                 & 1.001e-2         \\
                                                   &                                  & TFDL                                & 100\%                                & 9.478e-3         \\
                                                   &                                  & TDDL                                  & 100\%                                & 7.194e-3         \\
                \hline
            \end{tabular}
        }
    }
\end{table}

\begin{figure*}
    \centering
    \includegraphics[width=0.6\textwidth]{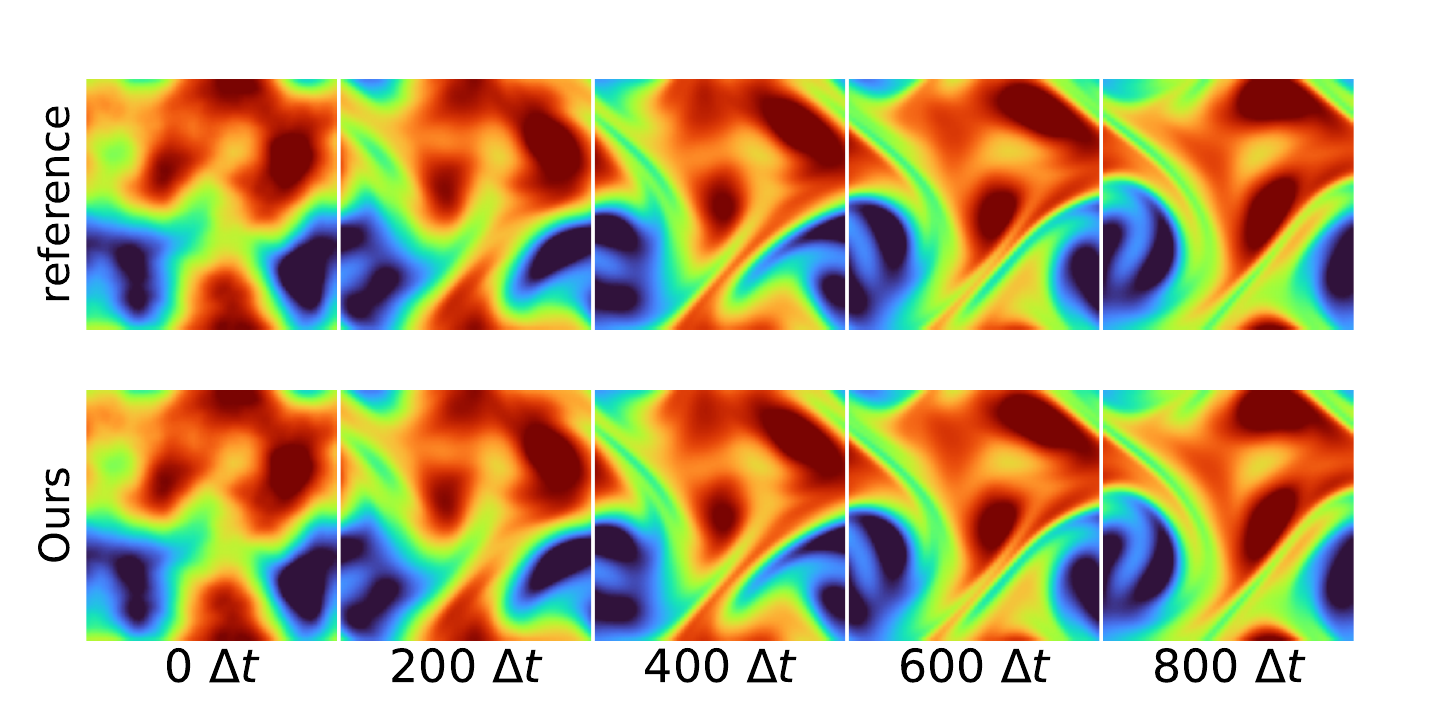}
    \includegraphics[width=0.6\textwidth]{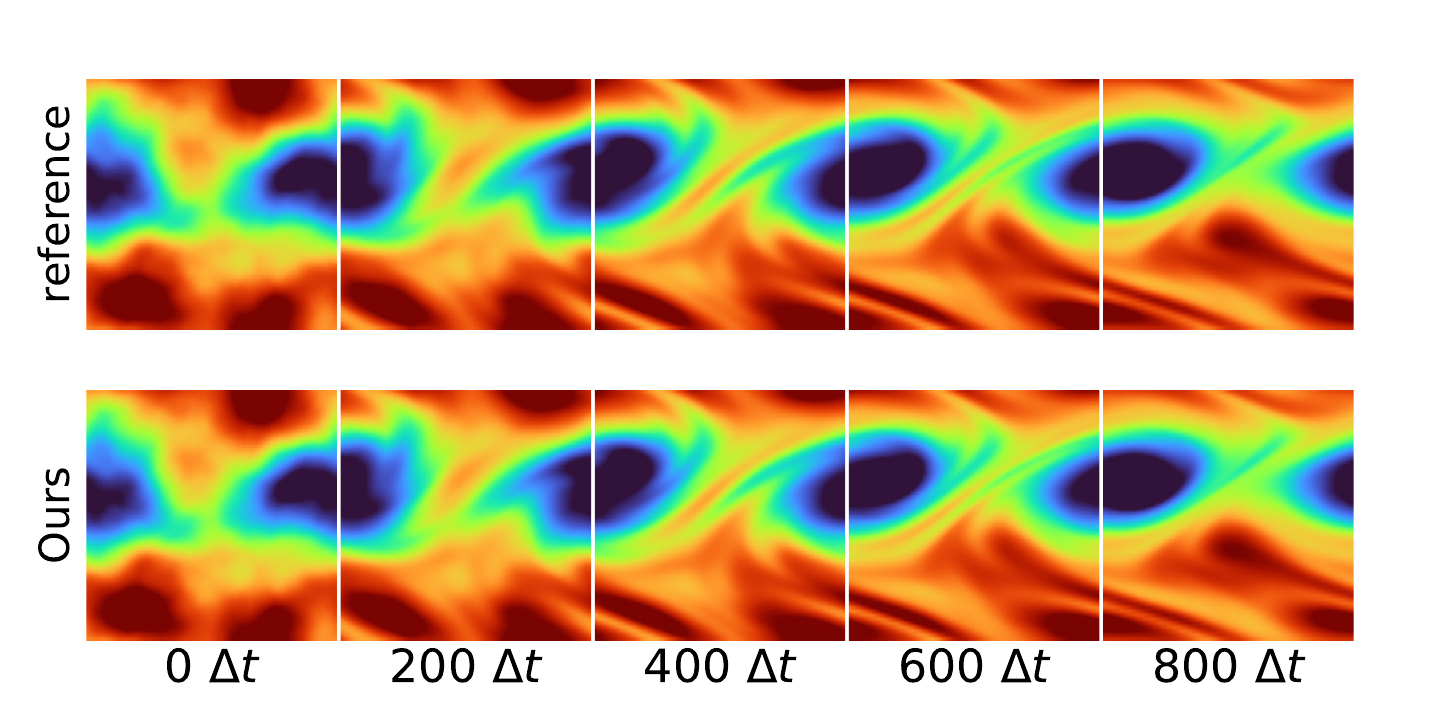}
    \includegraphics[width=0.6\textwidth]{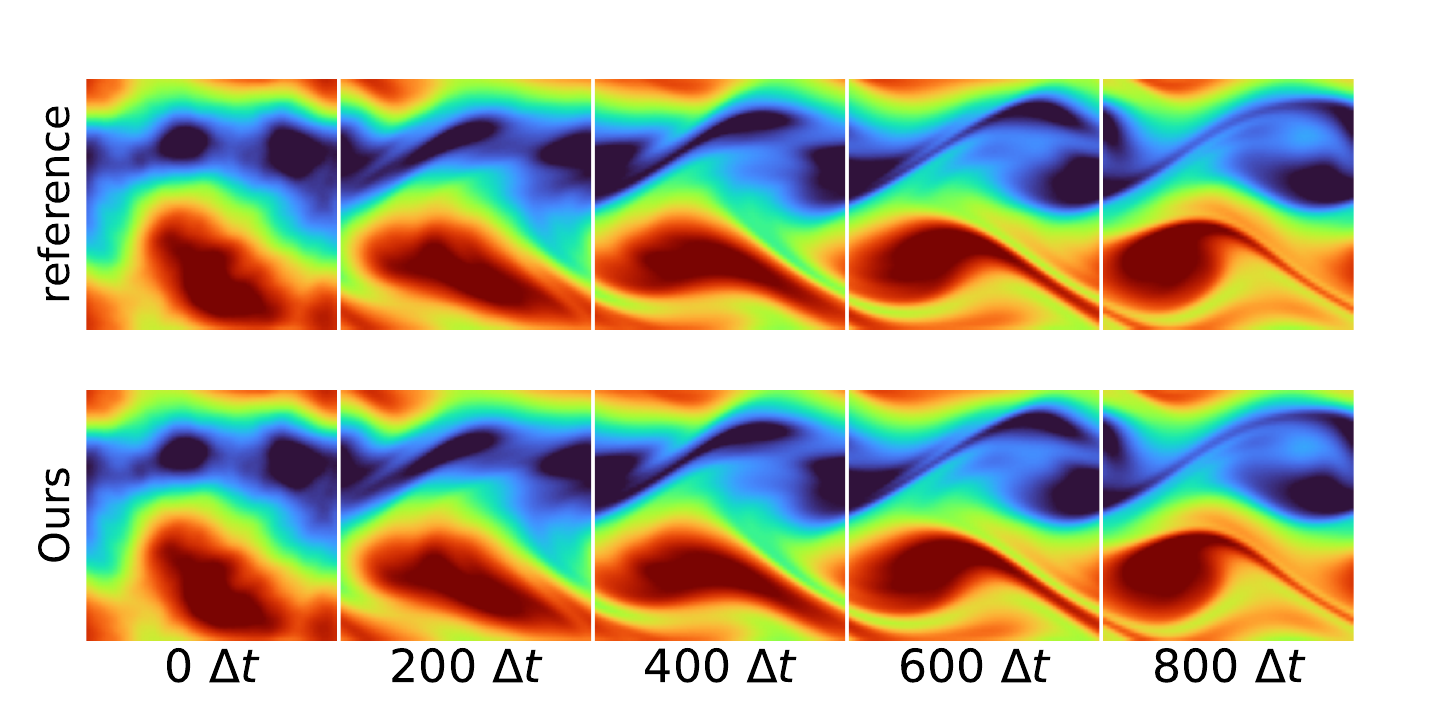}
    \includegraphics[width=0.6\textwidth]{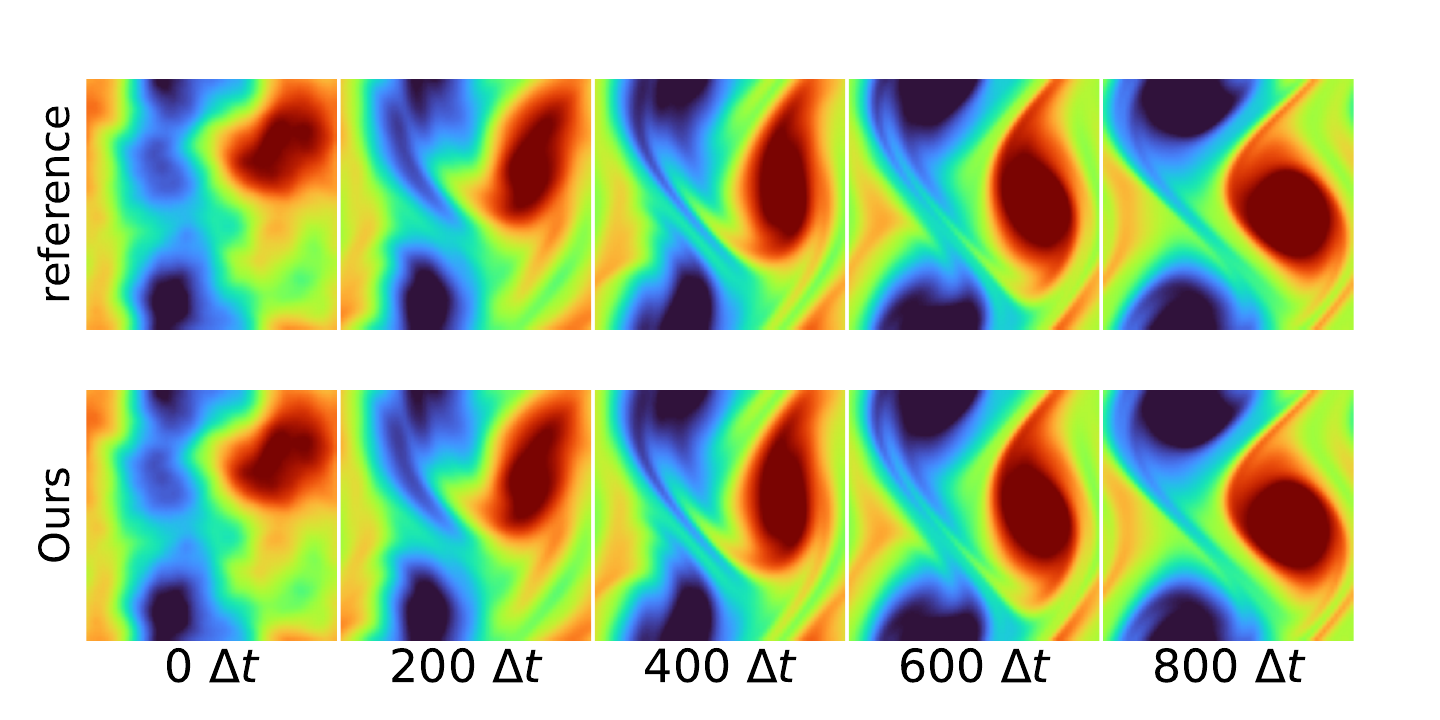}
    \caption{\textbf{NS equation:} Predicted snapshots at different time steps. `reference' and `Ours' stand for the ground truth and the corresponding predictions of PDE-Net++ (TDDL, FNO), respectively.
    }
    \label{fig:ns_result}
\end{figure*}

\end{document}